\def\isarxiv{1}

\def\paperTitle{Fundamental Limits of Crystalline Equivariant Graph Neural Networks: A Circuit Complexity Perspective}

\def\paperAuthor{
Yang Cao\thanks{\texttt{ycao4@wyomingseminary.org}. Wyoming Seminary.}
\and
Zhao Song\thanks{\texttt{magic.linuxkde@gmail.com}. University of California, Berkeley.}
\and
Jiahao Zhang\thanks{\texttt{ml.jiahaozhang02@gmail.com}.}
\and
Jiale Zhao\thanks{\texttt{zh2841871831@gmail.com}. Guangdong University of Technology.}
}

\ifdefined\isarxiv
\documentclass[11pt]{article}
\usepackage[numbers]{natbib}
\else
\documentclass{article} 
\usepackage{iclr2026_conference,times}

\fi

\ifdefined\isarxiv

\usepackage{amsmath}
\usepackage{amsthm}
\usepackage{amssymb}
\usepackage{algorithm}
\usepackage{subfig}
\usepackage{algpseudocode}
\usepackage{graphicx}
\usepackage{grffile}
\usepackage{wrapfig,epsfig}
\usepackage{url}
\usepackage{xcolor}
\usepackage{epstopdf}

\usepackage{bbm}
\usepackage{dsfont}

\else 

\usepackage{hyperref}
\usepackage{url}

\fi

\ifdefined\isarxiv

\else

\usepackage{amsmath}
\usepackage{amsthm}
\usepackage{amssymb}
\usepackage{algorithm}
\usepackage{subfig}
\usepackage{algpseudocode}
\usepackage{graphicx}
\usepackage{grffile}
\usepackage{wrapfig,epsfig}
\usepackage{url}
\usepackage{xcolor}
\usepackage{epstopdf}

\usepackage{bbm}
\usepackage{dsfont}

\fi
 
\allowdisplaybreaks

\ifdefined\isarxiv

\usepackage{tikz}
\usepackage{hyperref}  
\hypersetup{colorlinks=true,citecolor=blue,linkcolor=blue} 
\usetikzlibrary{arrows}
\usepackage[margin=1in]{geometry}

\else
\fi
 
\graphicspath{{./figs/}}

\theoremstyle{plain}
\newtheorem{theorem}{Theorem}[section]
\newtheorem{lemma}[theorem]{Lemma}
\newtheorem{definition}[theorem]{Definition}

\newtheorem{fact}[theorem]{Fact}
\newtheorem{remark}[theorem]{Remark}

\renewcommand{\cite}{\citep}

\newcommand{\R}{\mathbb{R}}

\newcommand{\intdiv}{\mathbin{/\!/}}
\newcommand{\float}[2]{\left\langle #1, #2\right\rangle}

\DeclareMathOperator*{\Z}{\mathbb{Z}}

\DeclareMathOperator{\poly}{poly}

\begin{document}

\ifdefined\isarxiv

\date{}
\title{\paperTitle}
\author{\paperAuthor}

\else

\title{\paperTitle}

\author{Antiquus S.~Hippocampus, Natalia Cerebro \& Amelie P. Amygdale \thanks{ Use footnote for providing further information
about author (webpage, alternative address)---\emph{not} for acknowledging
funding agencies.  Funding acknowledgements go at the end of the paper.} \\
Department of Computer Science\\
Cranberry-Lemon University\\
Pittsburgh, PA 15213, USA \\
\texttt{\{hippo,brain,jen\}@cs.cranberry-lemon.edu} \\
\And
Ji Q. Ren \& Yevgeny LeNet \\
Department of Computational Neuroscience \\
University of the Witwatersrand \\
Joburg, South Africa \\
\texttt{\{robot,net\}@wits.ac.za} \\
\AND
Coauthor \\
Affiliation \\
Address \\
\texttt{email}
}

%

\newcommand{\fix}{\marginpar{FIX}}
\newcommand{\new}{\marginpar{NEW}}

\maketitle

\fi

\ifdefined\isarxiv
\begin{titlepage}
  \maketitle
  \begin{abstract}
    
Graph neural networks (GNNs) have become a core paradigm for learning on relational data. In materials science, equivariant GNNs (EGNNs) have emerged as a compelling backbone for crystalline-structure prediction, owing to their ability to respect Euclidean symmetries and periodic boundary conditions. Despite strong empirical performance, their expressive power in periodic, symmetry-constrained settings remains poorly understood.
This work characterizes the intrinsic computational and expressive limits of EGNNs for crystalline-structure prediction through a circuit-complexity lens. We analyze the computations carried out by EGNN layers acting on node features, atomic coordinates, and lattice matrices, and prove that, under polynomial precision, embedding width $d=O(n)$ for $n$ nodes, $O(1)$ layers, and $O(1)$-depth, $O(n)$-width MLP instantiations of the message/update/readout maps, these models admit a simulation by a \emph{uniform} $\mathsf{TC}^0$ threshold-circuit family of polynomial size (with an explicit constant-depth bound). Situating EGNNs within $\mathsf{TC}^0$ provides a concrete ceiling on the decision and prediction problems solvable by such architectures under realistic resource constraints and clarifies which architectural modifications (e.g., increased depth, richer geometric primitives, or wider layers) are required to transcend this regime. The analysis complements Weisfeiler-Lehman style results that do not directly transfer to periodic crystals, and offers a complexity-theoretic foundation for symmetry-aware graph learning on crystalline systems.

  \end{abstract}
  \thispagestyle{empty}
\end{titlepage}

{\hypersetup{linkcolor=black}
\tableofcontents
}
\newpage

\else

\begin{abstract}

Graph neural networks (GNNs) have become a core paradigm for learning on relational data. In materials science, equivariant GNNs (EGNNs) have emerged as a compelling backbone for crystalline-structure prediction, owing to their ability to respect Euclidean symmetries and periodic boundary conditions. Despite strong empirical performance, their expressive power in periodic, symmetry-constrained settings remains poorly understood.
This work characterizes the intrinsic computational and expressive limits of EGNNs for crystalline-structure prediction through a circuit-complexity lens. We analyze the computations carried out by EGNN layers acting on node features, atomic coordinates, and lattice matrices, and prove that, under polynomial precision, embedding width $d=O(n)$ for $n$ nodes, $O(1)$ layers, and $O(1)$-depth, $O(n)$-width MLP instantiations of the message/update/readout maps, these models admit a simulation by a \emph{uniform} $\mathsf{TC}^0$ threshold-circuit family of polynomial size (with an explicit constant-depth bound). Situating EGNNs within $\mathsf{TC}^0$ provides a concrete ceiling on the decision and prediction problems solvable by such architectures under realistic resource constraints and clarifies which architectural modifications (e.g., increased depth, richer geometric primitives, or wider layers) are required to transcend this regime. The analysis complements Weisfeiler-Lehman style results that do not directly transfer to periodic crystals, and offers a complexity-theoretic foundation for symmetry-aware graph learning on crystalline systems.
\end{abstract}

\fi



\section{Introduction}

Graphs are a natural language for relational data, capturing entities and their interactions in domains ranging from molecules and materials~\cite{mbs+23} to social~\cite{slys21} and recommendation networks~\cite{yhc18}. Graph neural networks (GNNs) have consequently become a standard tool for learning on such data: the message-passing paradigm aggregates information over local neighborhoods to produce expressive node and graph representations that power tasks such as node/edge prediction and graph classification. This message-passing template (i.e., graph convolution followed by nonlinear updates) underlies many successful architectures and applications~\cite{jep+21,bdp+23}.

Recently, \emph{equivariant} graph neural networks (EGNNs)~\cite{shw21} have emerged as a promising direction for modeling crystalline structures in materials science. By respecting Euclidean symmetries and periodic boundary conditions, EGNNs encode physically meaningful inductive biases, enabling accurate predictions of structures, energies, and related materials properties directly from atomic coordinates and lattice parameters~\cite{shw+22,mbs+23}. In practice, E(3)/E($n$)-equivariant message passing and related architectures achieve strong performance while avoiding some of the computational burdens of higher-order spherical-harmonics pipelines~\cite{tsk+18,ls22}, and they have been adapted to periodic crystals~\cite{jhl+23,hdv+23}. Moreover, EGNN-style backbones are now widely used within crystalline generative models, including diffusion/flow-based approaches that model positions, lattices, and atom types jointly~\cite{jhl+23,ycm+23,zpz+23}. 
Despite this progress, fundamental questions about \emph{expressive power} remain. In particular, we ask:
\begin{quote}
\emph{What are the intrinsic computational and expressive limits of EGNNs for crystalline-structure prediction?}
\end{quote}

Prior theory for (non-equivariant) message-passing GNNs analyzes expressiveness through the lens of the Weisfeiler–Lehman (WL) hierarchy~\cite{xhl18,mrf19,mrm20}, establishing that standard GNNs are at most as powerful as 1-WL and exploring routes beyond via higher-order or subgraph-based designs~\cite{mrf19,mbh19,cmr21,qrg22}; other lines study neural models via circuit-complexity bounds. However, WL-style results focus on discrete graph isomorphism and typically abstract away continuous coordinates and symmetry constraints, while most existing circuit-complexity analyses target different architectures (e.g., Transformers~\cite{llzm24,cll+25_rope}). These differences make such results ill-suited to crystalline settings, where periodic lattices, continuous 3D coordinates, and E($n$)-equivariance are first-class modeling constraints. This motivates a tailored treatment of EGNNs for crystals.

In this paper, we investigate the \emph{fundamental expressive limits of EGNNs in crystalline-structure prediction}~\cite{ks22,jhl+23,mcsw24}. Rather than comparing against WL tests, we follow a circuit-complexity route~\cite{chi24,l25}: we characterize the computations performed by EGNN layers acting on node features, atomic coordinates, and lattice matrices, and we quantify the resources required to simulate these computations with uniform threshold circuits. Placing EGNNs within a concrete circuit class yields immediate implications for the families of decision or prediction problems such models can (and provably cannot) solve under realistic architectural and precision constraints. This perspective complements WL-style analyses and is naturally aligned with architectures, such as EGNNs, that couple graph structure with continuous, symmetry-aware geometric features.

{\bf Contributions.} Our contributions are summarized as follows:
\begin{itemize}
\item {\bf Formalizing EGNNs' structure.} We formalize the definition of EGNNs (Definition~\ref{def:egnn}).
\item {\bf Circuit-complexity upper bound for EGNNs.} Under polynomial precision, embedding width $d=O(n)$, $O(1)$ layers, and $O(n)$-width $O(1)$-depth MLP instantiations of the message/update/readout maps, we prove that the EGNN class from Definition~\ref{def:egnn} can be implemented using a \emph{uniform} $\mathsf{TC}^0$ circuit family (Theorem~\ref{sec:main_result}).
\end{itemize}

{\bf Roadmap.} 
In Section~\ref{sec:related_works}, we summarize the related works. 
In Section~\ref{sec:prelim}, we present the basic concepts and notations. 
In Section~\ref{sec:circuit}, we analyze the circuit complexity of components.
In Section~\ref{sec:main_result}, we present our main results.
Finally, in Section~\ref{sec:conclusion}, we conclude our work.
\section{Related Work}\label{sec:related_works}

\textbf{CSP and DNG in Materials Discovery}
Early methods for CSP and DNG approached materials discovery by generating a large pool of candidate structures and then screening them with high-throughput quantum mechanical calculations \cite{ks65} to estimate stability. Candidates were typically constructed through simple substitution rules \cite{wbm21} or explored with genetic algorithms \cite{goh06, pn11}.
Later, machine learning models were introduced to accelerate this process by predicting energies directly \cite{shw+22, mbs+23}.

To avoid brute-force search, generative approaches have been proposed to directly design materials \cite{cyjc20, ysd+21, nsc18}. Among them, diffusion models have gained particular attention, initially focusing on atomic positions while predicting the lattice with a variational autoencoder \cite{xfg+21}, and more recently modeling positions, lattices, and atom types jointly \cite{jhl+23, ycm+23, zpz+23}. Other recent advances incorporate symmetry information such as space groups \cite{hdv+23, jhl+24, cllw24}, leverage large language models \cite{fa23, gsm+24}, or employ normalizing flows \cite{wpi+22}.

\textbf{Flow Matching for Crystalline Structures}
Flow Matching \cite{lcb+23, tmh+23, dpnt23} has recently established itself as a powerful paradigm for generative modeling, showing remarkable progress across multiple areas. The initial motivation came from addressing the heavy computational cost of Continuous Normalizing Flows (CNFs) \cite{crbd18}, as earlier methods often relied on inefficient simulation strategies \cite{rgnl21, bcba+22}. This challenge inspired a new class of Flow Matching techniques \cite{av22, tfm+23, hbc23}, which learn continuous flows directly without resorting to simulation, thereby achieving much better flexibility. Recent study including \cite{ccl+25_form,cgl+25_homo,lss+25_hofar,lss+25_nrflow} explore Flow Matching in higher orders. Thanks to its straightforward formulation and strong empirical performance, Flow Matching has been widely adopted in large-scale generation tasks. For instance, \cite{dsf23} proposes a latent flow matching approach for video prediction that achieves strong results with far less computation. \cite{csy25} introduce a video generating method that use Flow Matching to learn the interpolation on the latent space. \cite{zlf+25} applies consistency flow matching to robotic manipulation, enabling efficient and fast policy generation. \cite{jbj24} develops a flow-based generative model for protein structures that improves conformational diversity and flexibility while retaining high accuracy. \cite{lww+24} introduces CrystalFlow, a flow-based model for efficient crystal structure generation. Overall, Flow Matching has proven to be an efficient tool for generative modeling across diverse modalities. \textit{Notably, EGNN-style backbones have become a de facto choice for crystalline structure generative modeling: diffusion- and flow-based pipelines pair symmetry-aware message passing with periodic boundary handling to jointly model positions, lattices, and compositions}~\cite{jhl+23, ycm+23, zpz+23, lww+24}. In these systems, the equivariant message-passing core supplies an inductive bias that improves sample validity and stability while reducing reliance on higher-order tensor features~\cite{shw21, hdv+23, jhl+24}.

{\bf Geometric Deep Learning.} 
Geometric deep learning, particularly geometrically equivariant Graph Neural Networks (GNNs) that ensure E(3) symmetry, has achieved notable success in chemistry, biology, and physics \cite{jep+21, bmb+21, bdp+23, mbs+23, qcw+23, zwh+25}. In particular, equivariant GNNs have demonstrated superior performance in modeling 3D structures \cite{cdg+21, tls+23}. Existing geometric deep learning approaches can be broadly categorized into four types: (1) Invariant methods, which extract features stable under transformations, such as pairwise distances and torsion angles \cite{ssk+18, ggmg20, gbg21}; (2) Spherical harmonics-based models, which leverage irreducible representations to process data equivariantly \cite{tsk+18, ls22}; (3) Branch-encoding methods, encoding coordinates and node features separately and interacting through coordinate norms \cite{jes+20, shw21}; (4) Frame averaging frameworks, which model coordinates in multiple PCA-derived frames and achieve equivariance by averaging the representations \cite{pab+21, dsh+23}.

While these architectures have pushed the boundaries of modeling geometric data in 3D structures, and advanced equivariant and invariant neural architectures in learning geometric data in chemistry, biology, and physics domains, the fundamental limitations of such architectures in crystalline structures still remain less explored. In this paper, we reveal the fundamental expressive capability limitation of equivariant GNNs via the lens of circuit complexity.

{\bf Circuit Complexity and Machine Learning.} 
Circuit complexity is a fundamental notion in theoretical computer science, providing a hierarchy of Boolean circuits with different gate types and computational resources~\cite{v99,ab09}. This framework has recently been widely used to analyze the expressiveness of machine learning models: a model that can be simulated by a weaker circuit class may fail on tasks requiring stronger classes. A central line of work applies circuit complexity to understand Transformer expressivity. Early studies analyzed two simplified theoretical models of Transformers:  and Average-Head Attention Transformers and SoftMax-Attention Transformers~\cite{lag+22,mss22,ms23}. Subsequent results have extended these analyses to richer Transformer variants, including those with Chain-of-Thought (CoT) reasoning~\cite{fzg+23,llzm24,ms24}, looped architectures~\cite{gps+23,lf24,sdl+25}, and Rotary Position Embeddings (RoPE)~\cite{cll+24_tensor_tc, cll+25_rope,ysw+25,cssz25}. Beyond Transformers, circuit complexity has also been applied to other architectures such as state space models (SSMs)~\cite{cll+25_mamba}, Hopfield networks~\cite{lll+24}, and various vision generative models, such as diffusion models~\cite{gkl+25, ccsz25,cll+25_var,kll+25_var} and autoregressive models~\cite{kll+25_tc}, as well as graph neural networks (GNNs)~\cite{g23,cgws24,lls+25}. In this work, we study the circuit complexity bounds of equivariant GNNs on crystalline structures, providing the first analysis of this kind.

{\bf Fundamental Limits of Neural Networks.}
A growing body of theoretical work seeks to describe the inherent limitations of neural networks, particularly Transformer-based architectures, in terms of their expressivity, statistical efficiency, and learnability. In the context of expressivity, recent studies establish the universal approximation abilities of various architectures, including prompt tuning Transformers~\cite{hwg+24}, attention mechanisms viewed as max-affine partitions~\cite{lhsl25}, and visual autoregressive Transformers~\cite{cll+25_var}.
Beyond expressivity, recent studies have characterized the statistical and computational trade-offs of large generative models, establishing provably efficient criteria for diffusion Transformers~\cite{hwsl24, hwl+24}. Meanwhile, several works identify inherent limitations of gradient-based optimization, demonstrating provable failures of Transformers in learning simple Boolean functions~\cite{hzs25+, cssz25}.
In addition, a series of works investigate the computational and architectural properties of modern Transformer variants, analyzing the fine-grained complexity of attention mechanisms~\cite{as23, as24_iclr, as24_neurips, as25_rope}, stability against rank collapse~\cite{as25_rank}, and efficient gradient computation~\cite{chl+24_rope, lss+24}. Related studies extend these analyses to LoRA fine-tuning~\cite{hsk+24}, modern Hopfield models~\cite{hlsl24}, and higher-order tensor attention architectures~\cite{lssz24_tat}.
A parallel research direction studies in-context learning as an emergent algorithmic capability of Transformers, analyzing its mechanisms through associative memory retrieval~\cite{wsh+24, hlzl25}, gradient-based multi-step updates in looped architectures~\cite{cll+25_icl}, and algorithm emulation~\cite{gsx23, csy23b, swxl24, wsh+24}.
Together, these efforts provide a comprehensive theoretical understanding of the limits and capabilities of modern neural networks, aligning with empirical findings that modern NNs, such as diffusion models~\cite{cgh+25,ghh+25,cgs+25,ghs+25_text} and language models~\cite{whs23,syz25,ghsz25,sma+25}, may still fail on simple problems.

\section{Preliminary} \label{sec:prelim}

We begin by introducing some basics of crystal representations in Section~\ref{sec:repr_crystal}, and then introduce the background knowledge of equivariant graph neural networks (EGNNs) in Section~\ref{sec:egnn}. In Section~\ref{sec:preli:circuit_complexity}, we present the fundamental concepts of circuit complexity.

\subsection{Representation of Crystal Structures} \label{sec:repr_crystal}

The unit cell representation describes the basis vectors of the unit cell, and all the atoms in a unit cell.

\begin{definition}[Unit cell representation of a crystal structure, implicit in page 3 of~\cite{jhl+23}] \label{def:unit_cell}
    Let $A:= [a_1, a_2, \ldots, a_n] \in \R^{h \times n}$ denote the set of description vectors for each atom in the unit cell. 
    Let $X:= [x_1, x_2 ,\ldots ,x_n] \in \R^{3 \times n}$ denote the list of Cartesian coordinates of each atom in the unit cell.
    Let $L := [l_1, l_2, l_3] \in \R^{3\times 3}$ denote the lattice matrix, where $l_1, l_2,l_3$ are linearly independent. The unit cell representation of a crystal structure is expressed by the triplet ${{\cal C}}:=(A,X,L)$. 
\end{definition}

The atom set representation describes a set containing an infinite number of atoms in the periodic crystal structure. 

\begin{definition}[Atom set representation of a crystal structure, implicit in page 3 of~\cite{jhl+23}] \label{def:atom_set}
    Let ${{\cal C}}:=(A,X,L)$ be a unit cell representation of crystal structure as Definition~\ref{def:unit_cell}, where $A:= [a_1, a_2, \ldots, a_n] \in \R^{h \times n}$, $X:= [x_1, x_2 ,\ldots ,x_n] \in \R^{3 \times n}$, and $L := [l_1, l_2, l_3] \in \R^{3\times 3}$. The atom set representation of ${\cal C}$ is defined as follows:
    \begin{align*}
        S({\cal C}) := \{(a, x): a = a_i, x = x_i + Lk, \forall i \in [n], \forall k \in \mathbb{Z}^{3}\},
    \end{align*}
    where $k$ is a length-$3$ column integer vector.
\end{definition}

\begin{definition}[Fractional coordinate matrix, implicit in page 3 of~\cite{jhl+23}]\label{def:frac_coor_mat}
      Let ${{\cal C}}:=(A,X,L)$ be a unit cell representation of crystal structure as Definition~\ref{def:unit_cell}, where $A:= [a_1, a_2, \ldots, a_n] \in \R^{h \times n}$, $X:= [x_1, x_2 ,\ldots ,x_n] \in \R^{3 \times n}$, and $L := [l_1, l_2, l_3] \in \R^{3\times 3}$. 
      We say that $F:= [f_1, f_2, \ldots, f_n] \in [0,1)^{3\times n}$ is a fractional coordinate matrix for ${{\cal C}}$ if and only if for all $i \in [n]$, we have:
      \begin{align*}
          x_i = Lf_i.
      \end{align*}
\end{definition}

\begin{definition}[Fractional unit cell view of a crystal structure, implicit in page 3 of~\cite{jhl+23}] \label{def:frac_unit_cell}
    Let ${{\cal C}}:=(A,X,L)$ be a unit cell representation of crystal structure as Definition~\ref{def:unit_cell}. Let $F$ be a fractional coordinate matrix as Definition~\ref{def:frac_coor_mat}. The fractional unit cell representation of ${\cal C}$ is a triplet ${\cal C}_{\mathrm{frac}}:= (A,F,L)$.
\end{definition}

\begin{fact}[Equivalence of unit cell representations, informal version of Fact~\ref{fac:equi_unit_cell_repr_formal}] \label{fac:equi_unit_cell_repr}
    For any fractional unit cell representation ${\cal C}_{\mathrm{frac}}:=(A,F,L)$ as Definition~\ref{def:frac_unit_cell}, there exists a unique corresponding non-fractional unit cell representation ${\cal C}:=(A,X,L)$ as definition~\ref{def:unit_cell}.
\end{fact}

Therefore, since both unit cell representations are equivalent, we only use the fractional unit cell representation in this paper. For notation simplicity, we may abuse the notation ${\cal C}$ to denote ${\cal C}_{\mathrm{frac}}$ in the following parts of this paper.

\begin{definition}[Fractional atom set representation of a crystal structure, implicit in page 3 of~\cite{mcsw24}]\label{def:frac_atom_set}
    Let ${\cal C}_{\mathrm{frac}}:= (A, F, L)$ be a fractional unit cell representation of a crystal structure as Definition~\ref{def:frac_unit_cell}, where $A:= [a_1, a_2,\ldots, a_n] \in \R^{h\times n}$, $F:= [f_1, f_2, \ldots, f_n] \in \R^{3\times n}$, and $L := [l_1, l_2, l_3] \in \R^{3\times 3}$. The atom set representation of ${\cal C}$ is defined as follows:
    \begin{align*}
        S_{\mathrm{frac}}({\cal C}) := \{(a, f): a = a_i, f = f_i + k, \forall i \in [n], \forall k \in \mathbb{Z}^{3}\},
    \end{align*}
    where $k$ is a length-$3$ column integer vector.
\end{definition}

\subsection{Equivariant Graph Neural Network Architecture} \label{sec:egnn}

We first define a useful transformation that computes the distance feature between each two atoms. 
\begin{definition}[$k$-order Fourier transform of relative fractional coordinates] \label{def:fourier_frac_coord}
    Let $x \in (-1 ,1)^3$ be a length-$3$ column vector. Without loss of generality, we let $k \in \mathbb{Z}_+$ be a positive even number.
    Let the output of the $k$-order Fourier fractional coordinates be a matrix $Y \in \R^{3 \times k}$ such that $Y:= \psi_{\mathrm{FT}, k}(x)$. For all $i \in [3], j \in [k]$, each element of $Y$ is defined as:
    \begin{align*}
        Y_{i,j}:= \begin{cases}
        \sin(\pi j x_i), & j\mathrm{
        ~is~even}; \\
        \cos(\pi j x_i), & j\mathrm{
        ~is~odd}.
    \end{cases}
    \end{align*}
\end{definition}

Then, we define a single layer for the Equivariant Graph Neural Network (EGNN) on the fractional unit cell representation of crystals.

\begin{definition}[Pairwise Message] \label{def:message}
    Let $\mathcal{C}:=(A,F,L)$ be a fractional unit cell representation as Definition~\ref{def:frac_unit_cell}, where $A \in \R^{h\times n}$, $F:= [f_1, f_2, \ldots, f_n] \in \R^{3\times n}$, and $L \in \R^{3\times 3}$. Let $H := [h_1, h_2, \ldots, h_n]\in \R^{d\times n}$ be a hidden neural representation for all the atoms. 
    Let $\psi_{\mathrm{FT},k}$ be a $k$-order Fourier transform of relative fractional coordinates as Definition~\ref{def:fourier_frac_coord}. 
    Let $\phi_{\mathrm{msg}}:\R^{d} \times \R^d \times \R^{3\times 3}\times \R^{3\times k} \to \R^d$ be an arbitrary function. We define the message $\mathsf{MSG}_{i,j}(F,L,H) \in \R^d$ between the $i$-th atom and the $j$-th atom for all $i, j \in [n]$ as follows:
    \begin{align*}
        \mathsf{MSG}_{i,j}(F,L,H):= \phi_{\mathrm{msg}}(h_i, h_j, L^\top L, \psi_{\mathrm{FT}, k}(f_i - f_j)).
    \end{align*}
\end{definition}

\begin{definition}[One EGNN layer]\label{def:one_egnn_layer}
    Let $\mathcal{C}:=(A,F,L)$ be a fractional unit cell representation as Definition~\ref{def:frac_unit_cell}, where $A:= [a_1, a_2,\ldots, a_n] \in \R^{h\times n}$, $F:= [f_1, f_2, \ldots, f_n] \in \R^{3\times n}$, and $L := [l_1, l_2, l_3] \in \R^{3\times 3}$. Let $H := [h_1, h_2, \ldots, h_n]\in \R^{d\times n}$ be a hidden neural representation for all the atoms. Let $\phi_{\mathrm{upd}}:\R^{d} \times \R^{d} \to  \R^d$ be an arbitrary function. 
    Let $\mathsf{MSG}$ be the message function defined as Definition~\ref{def:message}.
    Let the output of the $i$-th EGNN layer $\mathsf{EGNN}_i(A,F,L,H)$ be a matrix $Y = [y_1, y_2,\ldots,y_n] \in \R^{d\times n}$, i.e., $Y:= \mathsf{EGNN}_i(F,L,H)$. 
    For all $i \in [n]$, each column of $Y$ is defined as:
    \begin{align*}
        y_i := h_i + \phi_{\mathrm{upd}}(h_i, \sum_{j=1}^n \mathsf{MSG}_{i,j}(F, L, H)). 
    \end{align*}
\end{definition}

\begin{definition}[EGNN] \label{def:egnn}
    Let $\mathcal{C}:=(A,F,L)$ be a fractional unit cell representation as Definition~\ref{def:frac_unit_cell}, where $A\in \R^{h\times n}$, $F \in \R^{3\times n}$, and $L\in \R^{3\times 3}$. Let $q$ be the number of $\mathsf{EGNN}$ layers. Let $\phi_{\mathrm{in}}:\R^{h\times n} \to \R^{d\times n}$ be an arbitrary function for the input transformation. The $q$-layer $\mathsf{EGNN}:\R^{d\times n}\times \R^{3\times n} \times \R^{3\times 3} \to \R^{d\times n}$ can be defined as follows:
    \begin{align*}
        \mathsf{EGNN}(A,F,L) := \mathsf{EGNN}_q \circ \mathsf{EGNN}_{q-1} \circ\cdots \circ \mathsf{EGNN}_1(\phi_{\mathrm{in}}(A), F, L).
    \end{align*}
\end{definition}

\begin{remark}
    While functions $\phi_{\mathrm{msg}}$, $\phi_{\mathrm{upd}}$, and $\phi_{\mathrm{in}}$ are usually implemented as simple MLPs in practice, our theoretical result on equivariance and invariance works for any possible instantiation of these functions.
\end{remark}

\subsection{Class of Circuit Complexity}\label{sec:preli:circuit_complexity}
Here, we present key preliminaries and Boolean circuits for circuit complexity.

\begin{definition}[Boolean Circuit, page 102 on~\cite{ab09}]\label{def:boolean_circuit}
For a positive integer $n$, a Boolean circuit can be represented as a directed acyclic graph whose vertices, called gates, implement a mapping from $n$-bit binary strings to a single bit. Gates without incoming connections serve as the inputs, corresponding to the $n$ binary inputs. The remaining gates evaluate a Boolean function on the outputs from their preceding gates.
\end{definition}

Since each circuit is limited to inputs of a predetermined length, we consider a sequence of circuits is employed to handle languages that comprise strings of varying lengths.

\begin{definition}[Recognition of languages by circuit families, page 103 on~\cite{ab09}]\label{def:circuit_recognize_languages}
Consider a language $L \subseteq \{0,1\}^*$ and a family of Boolean circuits $C = \{C_n\}_{n \in \mathbb{N}}$, $C$ is said to recognize $L$ if, for each string $x$ over $(0, 1)$,
    $C_{|x|}(x) = 1 \iff x \in L.$
\end{definition}

Imposing bounds on circuits allows us to define certain classes of complexity, for example $\mathsf{NC}^i$.

\begin{definition}[$\mathsf{NC}^i$, page 40 on~\cite{ab09}]\label{def:nc}
A language is in $\mathsf{NC}^i$ if their exists a Boolean circuit family that decides it, using at most $O((\log n)^i)$ depth, polynomial size $O(\mathrm{poly}(n))$, and composed of bounded fan-in $\mathsf{AND}$, $\mathsf{OR}$, and $\mathsf{NOT}$ operations.
\end{definition}

By allowing $\mathsf{AND}$ and $\mathsf{OR}$ gates to have unlimited fan-in, we obtain more expressive circuits, which define the class $\mathsf{AC}^i$.

\begin{definition}[$\mathsf{AC}^i$,~\cite{ab09}]\label{def:ac}
A language is in $\mathsf{AC}^i$ if their exists a Boolean circuit family that decides it using polynomial many gates $O(\mathrm{poly}(n))$, depth bound by $O((\log n)^i)$, and built from $\mathsf{OR}$, $\mathsf{NOT}$, and $\mathsf{AND}$ gates, with $\mathsf{AND}$ gates and $\mathsf{OR}$ gates may take arbitrarily many inputs.
\end{definition}

Given that $\mathsf{MAJORITY}$ gates are capable of implementing $\mathsf{NOT}$, $\mathsf{OR}$ and $\mathsf{AND}$, an even larger class $\mathsf{TC}^i$ can be defined.

\begin{definition}[$\mathsf{TC}^i$,~\cite{ab09}]\label{def:tc}
A language is in $\mathsf{TC}^i$ if a Boolean circuit family exists that recognizes it using polynomial many gates $O(\mathrm{poly}(n))$, depth $O((\log n)^i)$, and consisting of $\mathsf{OR}$, $\mathsf{NOT}$ and $\mathsf{AND}$, and unbounded fan-in $\mathsf{MAJORITY}$ gates, where each $\mathsf{MAJORITY}$ gate returns $1$ when the number of $1$s among its inputs exceeds the number of $0$s.
\end{definition}

\begin{remark}
According to Definition~\ref{def:tc}, the $\mathsf{MAJORITY}$ gates of $\mathsf{TC}^i$ is able to substituted with $\mathsf{MOD}$ or $\mathsf{THRESHOLD}$ gates. Circuits that employ such gates are referred to as threshold circuits.
\end{remark}

\begin{definition}[$\mathsf{P}$, implicit in page 27 on~\cite{ab09}]
A language is in $\mathsf{P}$ if their a deterministic Turing machine exists which determines membership within polynomial time.
\end{definition}

\begin{fact}[Hierarchy folklore,~\cite{ab09,v99}]
The following class inclusions are valid for all $i \geq 0$:
\begin{align*}
    \mathsf{NC}^i \subseteq \mathsf{AC}^i \subseteq \mathsf{TC}^i \subseteq \mathsf{NC}^{i+1} \subseteq \mathsf{P}.
\end{align*}
\end{fact}

\begin{definition}[$L$-uniform, following~\cite{ab09}]\label{def:l_uniformity}
We call $C = \{C_n\}_{n \in \mathbb{N}}$ L-uniform when there exists a Turing machine can produce $C_n$, given $1^n$ input while using space $O(\log n)$. Language $L$ is a member of an L-uniform class such as $\mathsf{NC}^i$ when it can be recognized by an L-uniform circuit family ${C_n}$ satisfying the requirements of $\mathsf{NC}^i$.
\end{definition}

Then, we introduce a stronger notion regarding uniformity defined in terms of a time bound.

\begin{definition}[$\mathsf{DLOGTIME}$-uniform]\label{def:dlogtime_uniform}
A sequence of circuits $C = \{C_n\}_{n \in \mathbb{N}}$ is called $\mathsf{DLOGTIME}$-uniform when there exists a Turing machine that produces a description of $C_n$ within $O(\log n)$ time given input $1^n$. A language belongs to $\mathsf{DLOGTIME}$-uniform class if there exists such a sequence of circuits that recognizes it while also meeting the required size and depth bounds.
\end{definition}

\subsection{Floating-point numbers}
In this subsection, we present the foundational concepts of floating-point numbers ($\mathsf{FPN}$s) and operations, which provide the computational basis for executing GNNs efficiently on practical hardware. 

\begin{definition}[Floating-point numbers,~\cite{chi24}]\label{dfn:fp_num}
A floating-point number with $p$-bit can be expressed by two binary integers $s$ and $e$, where the mantissa $|s|$ takes values in $\{0\} \cup [2^{p-1}, 2^p)$, while the exponent $e$ lies within $[-2^p, 2^p - 1]$. The value of the $\mathsf{FPN}$ is calculated as $s \cdot 2^e$. When $e = 2^p$, the value corresponds to positive infinity or negative infinity, determined by the sign of $s$. We denote by $\mathbb{F}_p$ the set containing all $p$-bit floating-point numbers.
\end{definition}

\begin{definition}[Quantization,~\cite{chi24}]\label{dfn:rounding}
    Consider a real number $r\in\R$ be a real number with infinite precision. Its nearest $p$-bit representation is written as $\mathrm{round}_p(r)\in\mathbb{F}_p$. If two such representations are equally close, $\mathrm{round}_p(r)$ is defined as the one with an even significand. 
\end{definition}

Then, we introduce the key floating-point computations involved in producing neural network outputs.

\begin{definition}[Floating-point arithmetic,~\cite{chi24}]\label{dfn:fp_ops}
    Let $x, y \in \Z$. The operator $\intdiv$ is defined by:
    \begin{align*}
  x \intdiv y &:= \begin{cases}
    \frac{x}{y} & \text{when $\frac{x}{y}$ is an integer multiple of $\frac{1}{4}$} \\
    \frac{x}{y} + \frac{1}{8} & \text{otherwise.}
  \end{cases}
\end{align*}
Given two floating-points $\float{s_1}{e_1}$ and $\float{s_2}{e_2}$ with $p$-bits, we formulate their basis arithmetic operations on them as: 
    \begin{align*}
\mathrm{addition:} \float{s_1}{e_1} + \float{s_2}{e_2} &:= \begin{cases}
\mathrm{round}_p({\float{s_1 + s_2 \intdiv 2^{e_1-e_2}}{e_1}}) & \text{if $e_1 \ge e_2$} \\
\mathrm{round}_p({\float{s_1 \intdiv 2^{e_2-e_1} + s_2}{e_2}}) & \text{if $e_1 \le e_2$}
\end{cases} \\
\mathrm{multiplication:} \float{s_1}{e_1} \times \float{s_2}{e_2} &:= \mathrm{round}_p(\float{s_1s_2}{e_1+e_2}) \\
\mathrm{division:} \float{s_1}{e_1} \div \float{s_2}{e_2} &:= 
\mathrm{round}_p({\float{s_1 \cdot 2^{p-1} \intdiv s_2}{e_1-e_2-p+1}}) \\
\mathrm{comparison:} \float{s_1}{e_1} \le \float{s_2}{e_2} &\Leftrightarrow \begin{cases}
s_1 \le s_2 \intdiv 2^{e_1-e_2} & \text{if $e_1 \ge e_2$} \\
s_1 \intdiv 2^{e_2-e_1} \le s_2 & \text{if $e_1 \le e_2$.}
\end{cases}
\end{align*}
\end{definition}

Building on the previous definitions, we show that these basic operations can be efficiently executed in parallel using simple $\mathsf{TC}^0$ circuit constructions, as stated in the lemma below:

\begin{lemma}[Implementing $\mathsf{FPN}$ operations using $\mathsf{TC}^0$ circuits,~\cite{chi24}]\label{lem:fp_ops_tc0}
    Let $p$ be a positive integer representing the number of digits. Assume $p\leq \poly(n)$, we have the following holds:
    \begin{itemize}
        \item Basic Operations: Arithmetic operations ``$+$'', ``$\times$'', ``$\div$'', as well as comparison ($\leq$) between two floating-points with $p$-bit (see Definition~\ref{dfn:fp_num}), are realizable with uniform threshold circuits of $O(1)$-depth and polynomial size in $n$. We denote the maximal depth needed for these fundamental operations by $d_{\mathrm{std}}$.
        \item Repeated Operations: Multiplying $n$ $p$-bit floating-point numbers, as well as the aggregated sum of $n$ $p$-bit $\mathsf{FPN}$s (rounding performed post-summation) is computable via uniform threshold circuits with $\poly(n)$ size and  $O(1)$-depth. 
        Let $d_\otimes$ and $d_\oplus$ denote the maximum circuit depth  for multiplication and addition.
    \end{itemize}
\end{lemma}

Besides standard floating-point computations, certain advanced or composite operations can likewise be executed within $\mathsf{TC}^0$ circuits, shown in the lemmas below:

\begin{lemma}[Exponential function approximation within $\mathsf{TC}^0$,~\cite{chi24}]\label{lem:exp}
    Assuming the precision $p$ is at most $\poly(n)$. Each floating-point input $x$ with $p$-bit can have its exponential $\exp(x)$ simulated by a uniform threshold circuit with polynomial size and fixed depth $d_{\mathrm exp}$, achieving a relative error no greater than $2^{-p}$.
\end{lemma}

\begin{lemma}[Computing the square root operation in $\mathsf{TC}^0$,~\cite{chi24}]\label{lem:sqrt}
     Assuming the precision satisfies  $p$ is at most $\poly(n)$. Any floating-point input $x$ with $p$-bit can have its square root computed using a uniform threshold circuit of gate complexity $O(\mathrm{poly}(n))$ and bounded depth $d_\mathrm{sqrt}$, achieving a relative error no greater than $2^{-p}$.
\end{lemma}

\begin{lemma}[Matrix multiplication realizable in $\mathsf{TC}^0$ circuits,~\cite{cll+25_rope}]\label{lem:mat_prod_tc0}
Let $A\in\mathbb{F}_p^{n_1\times n_2}$ and $B\in \mathbb{F}_p^{n_2\times n_3}$ be two matrix operands. If $p\leq\poly(n)$ and $n_1, n_2, n_3\leq n$, then there exists a polynomial size uniform threshold circuit, having depth no greater than $(d_{\mathrm{std}} + d_{\oplus})$, that performs the computation of the matrix product $AB$.
\end{lemma}

\section{Circuit Complexity of Crystalline EGNNs} \label{sec:circuit}

Initially, we present the circuit complexity of basic EGNN building blocks in Section~\ref{sec:circuit_basic}, and then show the circuit complexity for EGNN layers in Section~\ref{sec:circuit_layer}.

\subsection{Circuit Complexity of Basic EGNN Building Blocks} \label{sec:circuit_basic}

We begin by introducing a useful lemma that introduces the $\mathsf{TC}^0$ computation of trigonometric functions.

\begin{lemma}[Evaluating trigonometric function computation within $\mathsf{TC}^0$,~\cite{cll+25_rope}] \label{lem:trig_tc0}
For $p$-bit floating-point numbers with precision $p \leq \poly(n)$, uniform threshold circuits of polynomial size and constant depth $8d_\mathrm{std}+d_\oplus+d_\otimes$ can produce approximations of $\sin(x)$ and $\cos(x)$, with a relative error no larger than $2^{-p}$.
\end{lemma}

Then, we show that $k$-order Fourier Transforms, a fundamental building block for Crystalline EGNN layers, can be computed by the $\mathsf{TC}^0$ circuits. 

\begin{lemma}[$k$-order Fourier Transform computation in $\mathsf{TC}^0$] \label{lem:k_ft_tc0}
    Assume $p \leq \poly(n)$ and $k = O(n)$. For any $p$-bit floating-point number $x$, the function $\psi_{\mathrm{Ft},k}(x)$ from Definition~\ref{def:fourier_frac_coord} can be approximated using a uniform threshold circuit of polynomial size and constant depth $10d_\mathrm{std}+d_\oplus+d_\otimes$, achieving a relative error no greater than $2^{-p}$.
\end{lemma}

\begin{proof}
    According to Definition~\ref{def:fourier_frac_coord}, for each $(i,j)\in[3]\times[k]$ there are two fixed cases:
    
    {\bf Case 1.} $j$ is even, then $Y_{i,j} := \sin(\pi j x_i)$.
    Computing $\pi j x_i$ uses $2d_\mathrm{std}$ depth and $\poly(n)$ size. Then, according to Lemma~\ref{lem:trig_tc0}, we need to use $8d_\mathrm{std}+d_\oplus+d_\otimes$ and $\poly(n)$ size for the $\sin$ operation. Thus, the total depth of this case is $10d_\mathrm{std}+d_\oplus+d_\otimes$, with a polynomial size bounded by $\poly(n)$.

    {\bf Case 2.} $j$ is odd, then $Y_{i,j} := \cos(\pi j x_i)$.
    Similar to case 1, the only difference is we need to use $\cos$ instead of $\sin$. According to Lemma~\ref{lem:trig_tc0}, $\cos$ takes $8d_\mathrm{std}+d_\oplus+d_\otimes$ depth and $\poly(n)$ size, which is same as $\sin$ in case 1. Thus, the total depth of this case is $10d_\mathrm{std}+d_\oplus+d_\otimes$, with a polynomial size bounded by $\poly(n)$.

    Since all $[3] \times [k]$ elements in $Y$ can be computed in parallel, thus we need $3k$ parallel circuit with $10d_\mathrm{std}+d_\oplus+d_\otimes$ depth to simulate the computation of $Y$. Since $k = O(n)$, thus we can simulate the computation using a circuit of size $\poly(n)$ and $10d_\mathrm{std}+d_\oplus+d_\otimes = O(1)$ depth. Thus $k$-order Fourier Transform is achievable via a $\mathsf{TC}^0$ uniform threshold circuit.
\end{proof}

We also show that MLPs are computable with uniform $\mathsf{TC}^0$ circuits. 

\begin{lemma}[MLP computation in $\mathsf{TC}^0$,~\cite{cll+25_rope}] \label{lem:mlp_tc0}
    Let the precision satisfy $p \leq \poly(n)$. Under this condition, an MLP layer of depth $O(1)$ and width $O(n)$ is realizable through a uniform threshold circuit with polynomial size $\poly(n)$, depth at most $2d_\mathrm{std} + d_{\oplus}$, while ensuring the relative error does not exceed $2^{-p}$.
\end{lemma}

\subsection{Circuit Complexity of EGNN Layer} \label{sec:circuit_layer}

\begin{lemma}[Pairwise Message computation in $\mathsf{TC}^0$.] \label{lem:message_tc0}
    Assume $p \leq \poly(n)$, $d = O(n)$ and $k = O(n)$. Assume $\phi_{\mathrm{msg}}$ is instantiated with $O(1)$ depth and $O(n)$ width MLP. For any $p$-bit floating-point number $x$, the function $\mathsf{MSG}(F,L,H)$ from Definition~\ref{def:message} is able to be approximated through a uniform threshold circuit of polynomial size and constant depth $ 13d_{\mathrm{std}} + 2d_\oplus + d_\otimes$, achieving a relative error no greater than $2^{-p}$.
\end{lemma}

\begin{proof}
    We first analyze the arguments in for the $\phi_\mathrm{msg}$ function.
    The first two arguments do not involve computation.
    The third argument $L^\top L$ involves one matrix multiplication. According to Lemma~\ref{lem:mat_prod_tc0}, we could compute the matrix multiplication using a circuit of $\poly(n)$ size and $d_\mathrm{std}+d_\oplus$ depth.

    In order to analyze the last argument $\psi_{\mathrm{FT},k}(f_i - f_j)$, we first analyze $f_i - f_j$, which takes $d_{\mathrm{std}}$ depth and constant size. Then, according to Lemma~\ref{lem:k_ft_tc0}, we can compute the $\psi_{\mathrm{FT},k}(\cdot)$ with circuit of $\poly(n)$ size and $10d_\mathrm{std}+d_\oplus+d_\otimes$ depth. Therefore, we can compute the last argument $\psi_{\mathrm{FT},k}(f_i - f_j)$ employing a circuit with $\poly(n)$ size and  $11d_\mathrm{std}+d_\oplus+d_\otimes$ depth.

    Next, since $d = O(n)$ and $k = O(n)$ according to Lemma~\ref{lem:mlp_tc0}, we can use circuit with $\poly(n)$ size and $2d_\mathrm{std} + d_{\oplus}$ to compute the $\phi_{\mathrm{msg}}(\cdot)$ function.

    Combining above, we can use circuit with $\poly(n)$ size and $2d_\mathrm{std} + d_{\oplus} + \max\{d_\mathrm{std}+d_\oplus, 11d_\mathrm{std}+d_\oplus+d_\otimes\} = 13d_{\mathrm{std}} + 2d_\oplus + d_\otimes
        = O(1)$
     depth to compute the pairwise message. Thus, pairwise message computation can be simulated by a $\mathsf{TC}^0$ uniform threshold circuit.
\end{proof}

\begin{lemma}[One $\mathsf{EGNN}$ layer approximation in $\mathsf{TC}^0$, informal version of Lemma~\ref{lem:one_egnn_layer_tc0_formal}] \label{lem:one_egnn_layer_tc0}
    Assume that the precision $p$ grows at most polynomially with $n$, $d = O(n)$ and $k = O(n)$. Assume $\phi_{\mathrm{msg}}$ and $\phi_{\mathrm{upd}}$ are instantiated with $O(1)$ depth and $O(n)$ width MLPs. For any $p$-bit floating-point number $x$, the function $\mathsf{EGNN}_i(A,F,H)$ from Definition~\ref{def:one_egnn_layer} is able to be approximated via a uniform threshold circuit of polynomial size and constant depth $16d_{\mathrm{std}} + 3d_\oplus + 2d_\otimes$, achieving a relative error no greater than $2^{-p}$.
\end{lemma}

\section{Main Results} \label{sec:main_result}

In this section, we present our main result, demonstrating that under some assumptions, the EGNN class defined in Definition~\ref{def:egnn} can be implemented by a uniform $\mathsf{TC}^0$ circuit family.

\begin{theorem}
    Under the conditions that the precision satisfies $p\leq\poly(n)$, the embedding size $d = O(n)$, and the networks has $q = O(1)$, $k = O(n)$ layers, and all the functions $\phi_{\mathrm{msg}}$, $\phi_{\mathrm{upd}}$, and $\phi_{\mathrm{in}}$ are instantiated with $O(1)$ depth and $O(n)$ width MLPs, then the equivariant graph neural network $\mathsf{EGNN}:\R^{d\times n}\times \R^{3\times n} \times \R^{3\times 3} \to \R^{d\times n}$ which defined in Definition~\ref{def:egnn} can be realized by a uniform $\mathsf{TC}^0$ circuit family.
  \end{theorem}

  \begin{proof}
  Since $d = O(n)$, according to Lemma~\ref{lem:mlp_tc0}, the computation of first argument ($\phi_{\mathrm{in}}(A)$) can be approximated by a circuit of $2d_{\mathrm{std}}+d_\oplus$ depth and $\poly(n)$ size. Last two arguments does not include computation.

  Then, according to Lemma~\ref{lem:one_egnn_layer_tc0}, for each $\mathsf{EGNN}$ layer, we need a circuit with $\poly(n)$ size and $16d_{\mathrm{std}} + 3d_\oplus + 2d_\otimes$ depth to simulate the computation.

  Combining results above, since there are $q$ serial layer of $\mathsf{EGNN}$, we need circuit of $\poly(n)$ size and 
  \begin{align*}
      q(16d_{\mathrm{std}} + 3d_\oplus + 2d_\otimes + 2d_{\mathrm{std}}+d_\oplus) = & ~ 
      q(18d_{\mathrm{std}} + 4d_\oplus + 2d_\otimes) \\
      = & ~ O(1)
  \end{align*}
   depth to simulate the $\mathsf{EGNN}$. Thus, the $\mathsf{EGNN}$ can be simulated by a $\mathsf{TC}^0$ uniform threshold circuit.
\end{proof}

\section{Conclusion} \label{sec:conclusion}

We studied the computational expressiveness of equivariant graph neural networks (EGNNs) for crystalline-structure prediction through the lens of circuit complexity. Under realistic architectural and precision assumptions—polynomial precision, embedding width $d=O(n)$, $q=O(1)$ layers, and $O(1)$-depth, $O(n)$-width MLP instantiations of the message, update, and readout maps—we established that an EGNN as formalized in Definition~\ref{def:egnn} admits a simulation by a \emph{uniform} $\mathsf{TC}^0$ circuit family of polynomial size. Our constructive analysis further yields an explicit depth bound of
$q(18d_{\mathrm{std}} + 4d_{\oplus} + 2d_{\otimes})$,
thereby placing a concrete ceiling on the computations performed by such models.





\newpage
\onecolumn
\appendix

\begin{center}
    \textbf{\LARGE Appendix }
\end{center}


{\bf Roadmap.}
Section~\ref{sec:proof_prelim} provides the proofs omitted from Section~\ref{sec:prelim}. Section~\ref{sec:proof_circuit} presents the proofs that were left out in Section~\ref{sec:circuit}.

\section{Missing Proofs in Section~\ref{sec:prelim}} \label{sec:proof_prelim}

\begin{fact}[Equivalence of unit cell representations, formal version of Fact~\ref{fac:equi_unit_cell_repr}] \label{fac:equi_unit_cell_repr_formal}
    For any fractional unit cell representation ${\cal C}_{\mathrm{frac}}:=(A,F,L)$ as Definition~\ref{def:frac_unit_cell}, there exists a unique corresponding non-fractional unit cell representation ${\cal C}:=(A,X,L)$ as definition~\ref{def:unit_cell}.
\end{fact}
\begin{proof}
{\bf Part 1: Existence.} By Definition~\ref{def:unit_cell}, we can conclude that $L$ is invertible since all the columns in $L$ are linearly independent. Thus, we can choose $X = L^{-1} F$ and finish the proof.

{\bf Part 2: Uniqueness.} We show this by contradiction. First, we assume that there exist two different unit cell representations ${\cal C}_1:=(A,X_1,L)$ and ${\cal C}_2:=(A,X_2,L)$ for ${\cal C}_{\mathrm{frac}}$, i.e., $X_1 \neq X_2$. By Definition~\ref{def:frac_coor_mat}, we have $X_1 = X_2 = LF$, which contradicts $X_1 \neq X_2$. Hence, the proof is complete.
\end{proof}

\section{Missing Proofs in Section~\ref{sec:circuit}} \label{sec:proof_circuit}

\begin{lemma}[One $\mathsf{EGNN}$ layer approximation in $\mathsf{TC}^0$, formal version of Lemma~\ref{lem:one_egnn_layer_tc0}] \label{lem:one_egnn_layer_tc0_formal}
    Assume that the precision $p$ grows at most polynomially with $n$, $d = O(n)$ and $k = O(n)$. Assume $\phi_{\mathrm{msg}}$ and $\phi_{\mathrm{upd}}$ are instantiated with $O(1)$ depth and $O(n)$ width MLPs. For any $p$-bit floating-point number $x$, the function $\mathsf{EGNN}_i(A,F,H)$ defined in Definition~\ref{def:one_egnn_layer} is able to be approximated through a uniform threshold circuit of polynomial size and constant depth $16d_{\mathrm{std}} + 3d_\oplus + 2d_\otimes$, achieving a relative error no greater than $2^{-p}$.
\end{lemma}

\begin{proof}
  We start with analyzing the arguments in $\phi_{\mathrm{upd}}(\cdot)$. The first argument does not involve computation. For the second argument, according to Lemma~\ref{lem:message_tc0}, we need circuit with $\poly(n)$ size and $13d_{\mathrm{std}} + 2d_\oplus + d_\otimes$ depth to simulate $\mathsf{MSG}_{i,j}(F,L,H)$ computation.

  Then, for the summation $\sum_{j=1}^{n}\mathsf{MSG}_{i,j}(F,L,H)$, we can compute $n$ $\mathsf{MSG}_{i,j}(F,L,H)$ in parallel, and use a circuit with $d_\oplus$ width to perform the summation. Thus we can simulate the last argument with circuit of $poly(n)$ size $13d_{\mathrm{std}} + 2d_\oplus + 2d_\otimes$ depth to simulate the last argument.

  Next, for $\phi_{\mathrm{upd}}(\cdot)$, since $d = O(n)$, according to Lemma~\ref{lem:mlp_tc0}, we can simulate $\phi_{\mathrm{upd}}(\cdot)$ with circuit of $\poly(n)$ size $2d_{\mathrm{std}}+d_{\oplus}$ depth.
  Finally, for the addition of $\R^d$ size vector, we need circuit $\poly(n)$ size and $d_{\mathrm{std}}$ depth to simulate it.

  Combining circuits above, we can simulate $\mathsf{EGNN}_i(A,F,H)$ with a circuit of $\poly(n)$ size and 
  \begin{align*}
    13d_{\mathrm{std}} + 2d_\oplus + 2d_\otimes + 2d_{\mathrm{std}}+d_\oplus d_{\mathrm{std}} = & ~ 16d_{\mathrm{std}} + 3d_\oplus + 2d_\otimes \\
    = & ~ O(1)    
  \end{align*} 
  depth to simulate the computation. Thus, one $\mathsf{EGNN}$ layer can be simulated by a $\mathsf{TC}^0$ uniform threshold circuit.
\end{proof}

\newpage
\ifdefined\isarxiv
\bibliographystyle{alpha}
\bibliography{ref}
\else
\bibliographystyle{iclr2026_conference}
\bibliography{ref}
\fi

\end{document}